\documentclass[sigconf]{aamas}

\usepackage{balance} 
\usepackage{times}
\usepackage{helvet}
\usepackage{courier}
\usepackage{graphicx}
\usepackage{algorithm}
\usepackage{algorithmic}
\usepackage{geometry}
\usepackage{hyperref}
\usepackage{comment}
\usepackage{todonotes}
\usepackage{booktabs}
\usepackage{cleveref}
\usepackage{float}
\usepackage{stfloats}
\usepackage{placeins}

\newtheorem{theorem}{Theorem}

\newtheorem{definition}{Definition}

\newcommand{\Uadd}{U^{\text{add}}}
\newcommand{\Uorth}{U^{\text{orth}}}
\newcommand{\uorth}{u^{\text{orth}}}
\newtheorem{prop}{Proposition}

\newcommand{\R}{\mathbb{R}}
\newcommand{\T}{\top}

\newcommand{\ipV}[2]{\langle #1,#2\rangle_{V_t}}
\newcommand{\eps}{\varepsilon}
\usepackage{natbib}

\setcopyright{none}
\renewcommand\footnotetextcopyrightpermission[1]{}

\acmSubmissionID{1725}

\title{Learning to Price: Interpretable Attribute-Level Models for Dynamic Markets}

\author{Srividhya Sethuraman}
\affiliation{
  \institution{IIT Madras, TCS Research }
  \city{Chennai}
  \country{India}}
\email{cs23s024@cse.iitm.ac.in, srividhya.sethuraman1@tcs.com}

\author{Chandrashekar Lakshminarayanan}
\affiliation{
  \institution{IIT Madras}
  \city{Chennai}
  \country{India}}
\email{chandrashekar@cse.iitm.ac.in}

\begin{abstract}
Dynamic pricing in high-dimensional markets poses fundamental challenges of scalability, uncertainty, and interpretability. Existing low-rank bandit formulations learn efficiently but rely on latent features that obscure how individual product attributes influence price. We address this by introducing an interpretable \emph{Additive Feature Decomposition-based Low-Dimensional Demand (\textbf{AFDLD}) model}, where product prices are expressed as the sum of attribute-level contributions and substitution effects are explicitly modeled. Building on this structure, we propose \textbf{ADEPT} (Additive DEcomposition for Pricing with cross-elasticity and Time-adaptive learning) — a projection-free, gradient-free online learning algorithm that operates directly in attribute space and achieves a sublinear regret of $\tilde{\mathcal{O}}(\sqrt{d}T^{3/4})$. Through controlled synthetic studies and real-world datasets, we show that ADEPT (i) learns near-optimal prices under dynamic market conditions, (ii) adapts rapidly to shocks and drifts, and (iii) yields transparent, attribute-level price explanations. The results demonstrate that interpretability and efficiency in autonomous pricing agents can be achieved jointly through structured, attribute-driven representations.
\end{abstract}


\newcommand{\BibTeX}{\rm B\kern-.05em{\sc i\kern-.025em b}\kern-.08em\TeX}

\begin{document}


\pagestyle{plain}
\fancyhead{}


\maketitle 


\section{Introduction}
E-commerce platforms operate in complex, high-dimensional markets where product variety is vast and consumer behavior is dynamic. They rely on sophisticated tools for recommendation \cite{ricci2015recommender, koren2009matrix}, inventory-aware personalization \cite{jannach2016recommendation}, and demand forecasting \cite{taddy2019bayesian}. Among these tools, pricing remains a central and powerful lever — affecting not only revenue but also inventory movement, competition, and customer engagement. Dynamic pricing, the practice of adjusting prices over time in response to market conditions, enables platforms to align prices with consumer preferences, temporal trends, and competitive pressures \cite{chen2016empirical, elmaghraby2003dynamic}.

The dynamic pricing problem involves sequentially setting prices for products with the goal of maximizing cumulative revenue over time under uncertain and evolving demand. Dynamic pricing strategies can broadly be categorized into two paradigms. The first, demand-based pricing, adjusts product prices in response to aggregate demand signals, offering a coarse-grained control mechanism over revenue and inventory \cite{chen2016dynamic}. In contrast, personalized dynamic pricing leverages fine-grained user-level features — such as demographics, purchase history, and behavioral patterns — to set individualized prices aimed at maximizing user-specific revenue or engagement \cite{chen2019personalized,kawale2015price}. While demand-based methods are reactive and population-level, personalized approaches enable more nuanced, context-aware pricing decisions. 

In this paper, we focus on demand-based pricing. The demand for a product is typically decomposed into (i) a \textit{baseline} component—representing consumer preference over product attributes, and (ii) an \textit{elasticity} component that quantifies how demand changes in response to pricing \cite{berry1995automobile, hosken2008retail}. Self-price elasticity captures the effect of a product’s price on its own demand, while cross-price elasticity reflects the influence of other products’ prices, particularly for substitutes with similar features \cite{andrews2016estimating}. A key aspect in modeling elasticity is the substitution effect, wherein a similar product priced lower adversely affects the demand of a given product.

Dynamic pricing in real-world markets is challenging due to two key factors: (a) the \textit{high dimensionality} arising from large product catalogs and (b) the \textit{uncertainty} in baseline preferences and elasticity matrices.  An autonomous pricing agent that uses (a) a low-dimensional feature representation to tackle the high dimensionality of product catalogs and (b) an online learning paradigm to tackle the uncertainty in elasticity is a good candidate to solve the dynamic pricing problem. A very desirable quality expected from such an autonomous agent is interpretability: the decisions made by the agent need to offer business insights. 

In this paper, we study dynamic pricing in markets with closely related products—such as mobile phones, hotels, or airline tickets — where both demand and pricing are driven by product attributes. In this setting, we make two key contributions: (I) a novel interpretable low-dimensional demand model, and (II) a dynamic pricing algorithm tailored for this model under unknown and evolving market conditions.

\textbf{Contribution I:} Our notion of interpretability is the ability to capture substitution effects -- informally speaking, we want to answer the following question:
\begin{center}
\textbf{Why is product A priced higher than product B?}
\end{center}
We aim to develop an interpretable model that can answer:
\begin{center}
\emph{"Product A is priced higher because it offers additional features that product B lacks."}
\end{center}
To support such reasoning, we propose the Additive Feature Decomposition based Low-Dimensional (AFDLD) model, which expresses prices as additive contributions from observable attributes. By linking prices directly to measurable features (e.g., RAM, camera, display), AFDLD provides a transparent decomposition of each product’s price and captures substitution effects among similar items. This constitutes our first contribution: an interpretable framework that explains price differences and substitution patterns through attribute-level decomposition.

\textbf{Contribution II:} Next, under the AFDLD model, we address the question:
\begin{center}
\textbf{How should we price products in dynamic markets?}
\end{center}
Real-world markets are dynamic, with shifting demand and uncertain consumer preferences. To tackle this, we introduce ADEPT—an online learning algorithm for dynamic pricing under the AFDLD model.
ADEPT operates in attribute space and achieves sublinear regret of $\tilde{\mathcal{O}}(d^{1/2}T^{3/4})$, enabling effective pricing decisions in evolving environments.

\textbf{Novelty:} Mueller et al. \cite{mueller2019low}, introduce a low-rank demand and OPOK an online learning to solve the dynamic pricing problem. In comparison, our work is novel in two respects - (a) Our Additive Feature Decomposition (AFDLD) enables attribute level interpretability and (b) we explicitly incorporate substitution effects into the demand formulation. 
To the best of our knowledge, none of the existing literature on dynamic pricing explicitly models substitution effects across products. This omission is critical, as substitution drives cross-product interactions that significantly influence demand and revenue outcomes. Our work addresses this gap by incorporating substitution effects directly into the attribute-level pricing formulation, thereby enabling realistic and actionable pricing.

\noindent We next introduce Additive Feature Decomposition (AFD) to model interpretable, attribute-based prices.

\begin{figure}[!h]
    \centering
    \includegraphics[width=0.99\linewidth]{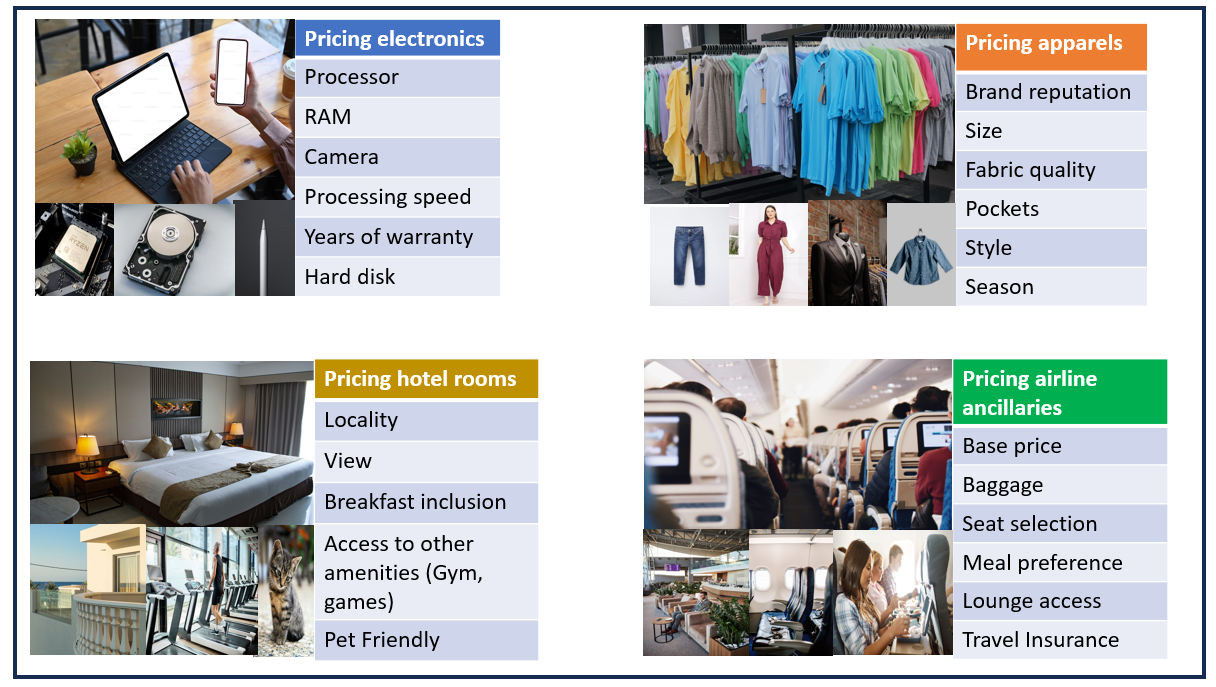}
   \caption{Illustration of additive feature-based price decomposition across domains.} 
    \label{fig:adddecompose}
\end{figure}

\section{Additive Feature Decomposition (AFD) of product prices} \label{sec:afd}
We begin with the definition of attribute-based features.

\begin{definition}[Product–Feature Matrix]
For a market with $N$ products, each having $d$ observable attributes, let $U \in \mathbb{R}^{N \times d}$ denote the product–feature matrix. The feature vector of the $i^{\text{th}}$ product is $U(i) = (U(i,1), \ldots,U(i,d))$, where for $j = 1, \ldots, d$, $U(i,j) \in \{0, \ldots, k\}$, i.e., each attribute takes a discrete value from $\{0, \ldots, k\}$. We also use $u(i)^\top \in \mathbb{R}^{d}$ to denote the $i^{\text{th}}$ row of $U \in \mathbb{R}^{N \times d}$.
\end{definition}

\paragraph{Additive feature decomposition (AFD)}
AFD prices a product as the \emph{sum of interpretable feature contributions}: each product/service attribute (size, brand, materials, options, policies) carries a transparent attribute level price. We now discuss how the Additive Feature Decomposition (AFD) occurs naturally in four different domains namely electronics,  apparels, hotel and airline markets (See Figure \ref{fig:adddecompose}). 

\emph{Examples:} 
\begin{enumerate}
    \item \textbf{Consumer electronics:} phone/tablet prices decompose into base hardware, storage tier, camera/display modules, and software/services premia (\Cref{fig:featrepr}). 
    \item \textbf{Apparel:} garments combine base fabric, workmanship, certification (e.g., organic), brand price, and seasonal exclusivity. 
    \item \textbf{Hotels:} a nightly rate splits into base room, location/view, board plan, amenities, and flexibility options. 
    \item \textbf{Airlines:} fares add base transport to taxes/fees, baggage, seat selection, lounge/service upgrades.
\end{enumerate}

\emph{Formal model.}
Let ${u(i)}^\top\in\mathbb{R}^d$ be the $i^{th}$ row of the feature matrix $U\in\mathbb{R}^{N\times d}$ denoting the features of the $i^{th}$ product. Let the attribute prices be denoted by $\theta\in\mathbb{R}^d$. Under AFD, the price $p(i)$ of a product i is given by,
\[
p(i) \;=\; \langle u(i),\theta\rangle \;=\; {u(i)}^\top \theta \;=\; \sum_{j=1}^d u(i,j)\theta(j),
\]
so each $u(i,j)\theta(j)$ is an explicit, controllable price addend. Here, each product price is directly expressed in terms of its attribute prices. This linear and interpretable structure underpins our algorithmic design: we learn (or adjust) $\theta$ online under box constraints, enabling scalable, real-time pricing.

\begin{figure}[!h]
    \centering
    \includegraphics[width=0.9\linewidth]{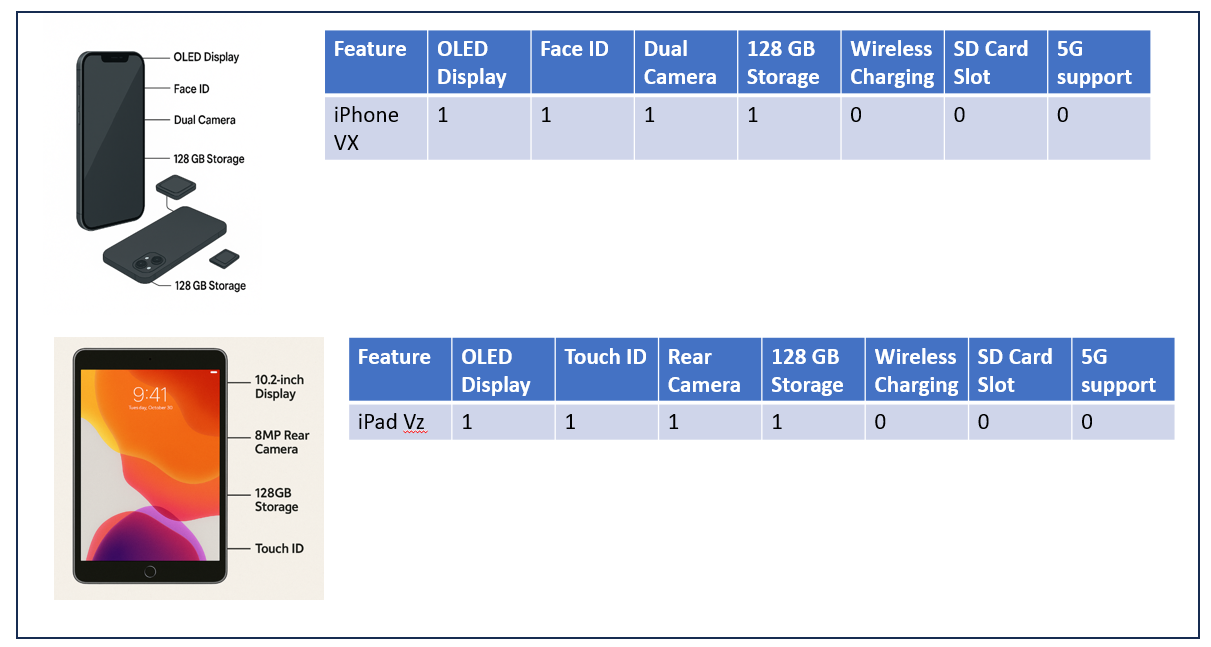}
   \caption{Illustration of AFD representation where interpretable attributes combine additively to form transparent, feature-level price components.}
     \label{fig:featrepr}
\end{figure}

\section{Related Work}
Most closely related work is that of \cite{mueller2019low}. We have briefly mentioned the novelty of our work over their work in the introduction. A detailed comparison will be made in \Cref{sec:attrelasmodel,sec:adept} where we propose our model and algorithm respectively. We now compare our work with other related work in the literature. 
Our work sits at the intersection of (i) dynamic pricing with contextual bandits, (ii) low-rank representations for generalization, and (iii) additive/attribute-based pricing. 

\textit{Dynamic pricing with contextual bandits.}
Contextual bandits model feature-dependent demand and have been applied to personalized and non-stationary pricing \citep{ban2021personalized,keskin2020dynamic,qi2021meta}. While these methods provide strong guarantees, they often struggle with scalability in high-dimensional feature spaces. We build on this line of work by introducing a structured, low-dimensional attribute representation that preserves sample efficiency. \newline
\textit{Low-rank and structured feature models.}
Low-rank structures improve generalization in contextual bandits and recommendations \citep{cohen2016contextual,sen2016contextual,zhu2021improving,wang2021low}. Unlike prior formulations that focus purely on compression, our model leverages a pricing-specific inductive bias: latent drivers correspond to interpretable product attributes. \newline
\textit{Additive and attribute-based pricing.}
Additive decomposition enables transparency and control in practical pricing systems \citep{chen2021interpretable,tang2022personalized,gupta2020pricing,aleksandrov2023modular}. Enterprise CPQ tools (e.g., Salesforce Communications Cloud) operationalize this through \emph{Attribute-Based Pricing} (ABP) rules \citep{SalesforceABPHelp,ZhangABPSSRN}. Our approach extends ABP into an \emph{online learning} framework, treating attribute adjustments as decision variables and learning from bandit feedback via projection-free updates. \newline

\section{AFDLD: A Novel Interpretable Low-Dimensional Demand Model} \label{sec:attrelasmodel}
In this section, we formulate the dynamic pricing problem as a regret minimization problem. We then introduce our low-dimensional demand model which is based on the additive features discussed in the previous section. We then show that under this novel demand model the regret minimization problem is a convex optimization problem. We then discuss how our novel model captures the substitution effects absent in the low-rank model \cite{mueller2019low}. We also illustrate the importance of capturing substitution effects via toy examples.

\textbf{Regret Formulation.} We study high-dimensional dynamic pricing where a seller offers $N$ products over $T$ time periods. At each round $t$, the seller chooses a price vector $p_t \in \mathbb{R}^N$, observes demand $q_t \in \mathbb{R}^N$, and earns revenue $R_t(p_t) = -\langle q_t, p_t \rangle$. The goal is to maximize cumulative revenue while keeping 
the \emph{regret} small:
\begin{align}\label{eq:regretmin}
\mathrm{Regret}(T) = \mathbb{E}\!\left[ \sum_{t=1}^T \big( R_t(p^*) - R_t(p_t) \big) \right],
\end{align}
where 
\[
p^* = \arg\min_{p \in \mathcal{P}} \sum_{t=1}^T- R_t(p), 
\]
is the best fixed price vector in hindsight and $\mathcal{P}$ is the set of allowable prices. Maximizing total revenue is equivalently minimizing total of negative revenue.

We now introduce our novel \textbf{Additive Feature Decomposition-based Low Dimensional (AFDLD) demand model }. As motivated in \Cref{sec:afd}, let $U \in \mathbb{R}^{N \times d}$ denote the product feature matrix defined in terms of the observable attributes, and $\theta_t \in [\theta_{\min}, \theta_{\max}]^d$ 
the attribute-level prices. Product-level prices are then given by
\[
p_t = U \theta_t
\]
hence additive decomposition ensures that each attribute contributes a distinct increment to the final price, providing transparency and interpretability.

Demand for product $i$ at time $t$ is modeled as
\begin{subequations}\label{eq:general-demand}
\begin{align}\label{eq:lddm}
q_t(i)
&= u(i)^\T z_t
   - \alpha_{ii}\,\ipV{u(i)}{u(i)}\,p_t(i) \notag\\
&\quad
   + \sum_{j=1}^N\alpha_{ij}\,\ipV{u(i)}{u(j)}\,(p_t(j) - p_t(i))
   + \eps_t(i).
\end{align}
\end{subequations}
where $z_t \in \mathbb{R}^d$ captures baseline demand in attribute space, $\alpha_{ii}$ encodes own-price sensitivity, 
and $\alpha_{ij}$ models substitution strength between product $i$ and $j$. 
The $(p_t(j) - p_t(i))$ term captures realistic substitution effects: if a close substitute $j$ becomes more expensive, then demand for $i$ rises and vice-versa. $V_t \in\mathbb{R}^{d \times d}$ is a positive definite matrix which captures the relative importance of attributes.

\paragraph{\textbf{Revenue as a Quadratic Form.}} Since the product price is related to the base price via $p_t=U\theta_t$, collecting terms yields a quadratic revenue function in $\theta_t$:  
\begin{equation}
R_t(\theta_t) \;=\; -\,\theta_t^\top {U}^\top M U \theta_t \;+\; \theta_t^\top {U}^\top z_t,
\end{equation}
where $M_t = M_1 + M_2 + M_3$ is the attribute-elasticity operator with
\begin{equation}
\begin{aligned}
M_1 &= 2\,\mathrm{Diag}({U}V_t{U^\top)}, \\
M_2 &= -\,UV_tU^\top, \\
M_3 &= \mathrm{Diag}\!\left(\sum_{j} \langle u(i), u(j) \rangle_{V_t} \right).
\end{aligned}
\end{equation}

\begin{prop}[Convex Optimization]
The objective $-R_t(\theta_t)$ is convex in $\theta_t$. 
\end{prop}
\textbf{Remarks:}

$\bullet$ For fixed $z_t=z,\forall t\geq 0 $ and $V_t=V, \forall t\geq 0 $ the market is static and $R_t(\theta)=R(\theta),\forall t\geq 0$. Since -$R(\theta)$ is convex, in the case of static market, we can compute the optimal attribute level price by minimizing as $\theta^*=\arg\min_{\theta}R(\theta)$. 

$\bullet$ In the case when the market is dynamic, minimizing the regret in \eqref{eq:regretmin} amounts to an \emph{online convex optimization problem}, which is achieved by our ADEPT algorithm (\Cref{sec:adept}).
\textit{(See Supplement for an extended discussion.)}\\

\textbf{Comparison with prior model.} In the low-rank demand model of \cite{mueller2019low} the demand for product $i$ is given by:
\begin{align}\label{eq:lrdm}
q_t(i) = \uorth_t(i)^\top z_t - \sum_{j=1}^N \langle \uorth(i),  \uorth(j) \rangle_{V_t} \, p_t(j) + \varepsilon_{t}(i)
\end{align}

We now list the differences between our model in \eqref{eq:lddm} and that of \cite{mueller2019low} in \eqref{eq:lrdm}, the key one being that \eqref{eq:lrdm} does not capture substitution effects.\\
$\bullet$ \textbf{Features:} In our model $U$ is based on observable product features, whereas in \eqref{eq:lrdm}, $\Uorth$ is an orthogonal latent feature matrix.

$\bullet$ \textbf{Baseline Demand:} In both models $z_t\in\R^d$ denotes the baseline demand. The difference is that in our model, $z_t$ captures the baseline demand for the observable product attributes, whereas in \eqref{eq:lrdm}, it stands for the baseline demand for the latent features.

$\bullet$ \textbf{Elasticity:} As discussed in the text below \eqref{eq:lddm}, our model captures the substitution effects. As can be seen from \eqref{eq:lrdm}, when two different products $i,j$ are similar, it is quite natural to expect that if $j$ is priced higher than $i$, customers will prefer $i$ more than $j$ which will lead to an increase in demand for $i$. However, in \eqref{eq:lrdm}, an increase in price of a similar product $j$ decreases the demand for product $i$ which is quite unnatural.

$\bullet$ \textbf{Low Dimension vs. Low Rank:} The final revenue expression in our setting is low-dimensional (see \eqref{eq:lddm}) and is not low-rank. In \eqref{eq:lrdm}, the model is low-rank, however, the price $p_t\in\R^N$ is not.

$\bullet$ \textbf{Pricing:} In this paper, we assume that the pricing is low-dimensional in that $p_t=U \theta_t$. In \eqref{eq:lrdm}, the pricing $p_t$ is still in $N$ dimensions which is addressed via a projection step in their algorithm (we discuss this in the next section).

\subsection{Interpretability of AFDLD Model}
In markets with similar products, substitution effects strongly influence both demand and pricing. 
Our notion of interpretability is therefore tied to these effects—understanding how prices and demand vary with attribute-level similarity among products. The AFDLD model captures this relationship explicitly. We now illustrate how it produces interpretable and transparent pricing behavior through simple synthetic examples that reveal its attribute-level reasoning.

\subsubsection*{Setting 1: Equal attribute importance in $\mathbf{V}$}
To illustrate interpretability under the AFDLD formulation, we consider a realistic electronics example involving three smartphone variants ($N{=}3$) and three observable attributes ($d{=}3$): storage tier ($A_1$), camera module ($A_2$), and display type ($A_3$). The product–attribute matrix $U \!\in\! \mathbb{R}^{3\times3}$ and baseline demand $z\!\in\!\mathbb{R}^3$ are explicitly specified, while the feature–interaction matrix $V{=}\mathbf{I}$ isolates the role of self-effects. The elasticity coefficients are fixed at $\alpha_{ii}=\alpha_{ij}=0.15$.

\paragraph{\textbf{Experiment 1: Non-overlapping attributes (Table \ref{tab:adept_exp_results1}).}} In this experiment, we consider a degenerate setting in which each product has only one attribute. One can think of product $P_1$ being storage of $128$ GB ($A_1$ alone), product $P_2$ being a dual camera ($A_2$ alone) and product $P_3$ being an OLED display ($A_3$ alone). In this case, for $i=1,2,3$ the features $U(i)$ of product $P(i)$ are given by:
\[
U(1) = [1 \; 0 \; 0], \quad
U(2) = [0 \; 1 \; 0], \quad
U(3) = [0 \; 0 \; 1].
\]

When baseline demand is identical, the learned $\theta^*$ values are identical across attributes, leading to equal product prices ($p{=}U\theta^*$). When baseline demand is non-identical, the estimated $\theta^*$ yields distinct prices in direct proportion to the attribute-specific demand.  Interpretively, the model reveals that the camera feature contributes most to willingness-to-pay, followed by storage and display—transparent from the magnitude of $\theta^*$.  Such one-to-one mapping between attribute value and price component enables clear managerial explanation: product 2 is costlier precisely because consumers value its camera feature more highly.

\begin{table}[H]
\centering
\caption{AFDLD model with $V = I$ and non-overlapping features.}
\label{tab:adept_exp_results1}
\renewcommand{\arraystretch}{1.3}
\begin{tabular}{@{}lccc@{}}
\toprule
\textbf{\begin{tabular}{@{}c@{}} Baseline \\ demand $z$ \end{tabular}} & \textbf{\begin{tabular}{@{}c@{}} $\theta^*$ \\ values\end{tabular}} & \textbf{\begin{tabular}{@{}c@{}} Final prices  \\ $p= U \theta^*$ \end{tabular}} \\ 
\midrule
 $[60, 60, 60]$ & 
\begin{tabular}[c]{@{}c@{}}[13.72,  13.70,  13.72]\end{tabular} & 
\begin{tabular}[c]{@{}c@{}}[13.72, 13.70,  13.72]\end{tabular} \\

$[100, 150, 250]$ & 
\begin{tabular}[c]{@{}c@{}}[22.96, 33.60, 56.96]\end{tabular} & 
\begin{tabular}[c]{@{}c@{}}[22.96, 33.60, 56.96]\end{tabular} \\
\bottomrule
\end{tabular}
\end{table}

\paragraph{\textbf{Experiment 2: Overlapping attributes (Table \ref{tab:adept_exp_results2}).}}
We now allow shared features across products,
\[
U(1) = [1 \; 1 \; 0], \quad
U(2) = [0 \; 1 \; 1], \quad
U(3) = [1 \; 0 \; 1].
\]
corresponding to P1 (Storage + Camera), P2 (Camera + Display), and P3 (Storage + Display). For identical baseline demand, we obtain nearly uniform $\theta^*$ values and thus similar final prices. Under heterogeneous demand, the optimal $\theta^*$ and resulting $p$ exhibit interpretable differentiation: the combination of high-value attributes (e.g., camera plus display) produces higher composite prices. Every price difference can be decomposed and explained through the additive structure $p_i=\langle U(i,\cdot),\theta^*\rangle$, exposing how much each feature contributes to each product’s final price.

\begin{table}[H]
\centering
\caption{AFDLD model with $V = I$ and overlapping features.}
\label{tab:adept_exp_results2}
\renewcommand{\arraystretch}{1.3}
\begin{tabular}{@{}lccc@{}}
\toprule
\textbf{\begin{tabular}{@{}c@{}} Baseline \\ demand $z$ \end{tabular}} & \textbf{\begin{tabular}{@{}c@{}} $\theta^*$ \\ values\end{tabular}} & \textbf{\begin{tabular}{@{}c@{}} Final prices  \\ $p= U \theta^*$ \end{tabular}} \\ 
\midrule
$[60, 60, 60]$ & 
\begin{tabular}[c]{@{}c@{}}[6.77, 6.84, 6.79]\end{tabular} & 
\begin{tabular}[c]{@{}c@{}}[13.61, 13.63, 13.56]\end{tabular} \\
 
$[100, 150, 250]$ & 
\begin{tabular}[c]{@{}c@{}}[14.69, 17.78, 24.30]\end{tabular} & 
\begin{tabular}[c]{@{}c@{}}[32.47, 42.08, 38.98]\end{tabular} \\ 
\bottomrule

\end{tabular}
\end{table}

\subsubsection*{Setting 2: Heterogeneous attribute importance in $\mathbf{V}$}
We now extend the smartphone example to study the effect of unequal attribute importance, captured by a diagonal interaction matrix $V = \mathrm{diag}([1.5,\,1.2,\,1.0])$. The attributes—storage tier (A$_1$), camera module (A$_2$), and display type (A$_3$)—now differ in relative market influence: camera upgrades are perceived as most critical, followed by storage and display. The elasticity coefficients are fixed at $\alpha_{ii}=\alpha_{ij}=0.15$, and the $U$ structures remain the same as in Setting 1 to isolate the role of $V$.  
\paragraph{\textbf{Experiment 1: Non-overlapping attributes (\Cref{tab:adept_diagV1}).}}
Here, we use the U matrix as in the previous setting with non-overlapping attributes. We observe that when the baseline demand is identical, the estimated $\theta^*$ produces proportional prices $p$. Here, even though demand is uniform, prices differ because the attributes vary in intrinsic importance encoded in $V$. This yields an interpretable ranking: OLED display (A$_3$) has the highest contribution, followed by camera and storage.  When the baseline demand is non-identical, the interaction of demand and attribute significance amplifies price separation. Retailers can directly interpret these coefficients as attribute-level prices: a higher baseline demand for the "'display-intensive'" product magnifies the already high display value.

\begin{table}[H]
\centering
\caption{AFDLD model with $V = \mathrm{diag}([1.5,\,1.2,\,1.0])$ with non-overlapping features.}
\label{tab:adept_diagV1}
\renewcommand{\arraystretch}{1.3}
\begin{tabular}{@{}lccc@{}}
\toprule
\textbf{\begin{tabular}{@{}c@{}} Baseline \\ demand $z$ \end{tabular}} & \textbf{\begin{tabular}{@{}c@{}} $\theta^*$ \\ values\end{tabular}} & \textbf{\begin{tabular}{@{}c@{}} Final prices  \\ $p= U \theta^*$ \end{tabular}} \\ 
\midrule
$[60, 60, 60]$ & 
\begin{tabular}[c]{@{}c@{}}[8.94, 11.40, 13.50]\end{tabular} & 
\begin{tabular}[c]{@{}c@{}}[8.94, 11.40, 13.50]\end{tabular} \\

 $[100, 150, 250]$ & 
\begin{tabular}[c]{@{}c@{}}[15.20, 27.92, 56.78]\end{tabular} & 
\begin{tabular}[c]{@{}c@{}}[15.19, 27.93, 56.78]\end{tabular} \\

\bottomrule
\end{tabular}
\end{table}

\paragraph{\textbf{Experiment 2: Overlapping attributes (\Cref{tab:adept_diagV2}).}}
To capture more realistic smartphones sharing multiple features, we use U matrix with overlapping attributes as in the previous setting corresponding to P1 (Storage + Camera), P2 (Camera + Display), and P3 (Storage + Display). When baseline demand is identical, we obtain  $\theta^*$ and prices $p$. Even under identical demand, attribute heterogeneity causes asymmetric prices, highlighting how AFDLD decomposes product value into interpretable feature-wise premiums. Under non-identical demand, the estimated $\theta^*$ and $p$ show that products combining more valued attributes (camera + display) consistently achieve higher prices. 

\begin{table}[H]
\centering
\caption{AFDLD model with $V = \mathrm{diag}([1.5,\,1.2,\,1.0])$ with overlapping features.}
\label{tab:adept_diagV2}
\renewcommand{\arraystretch}{1.3}
\begin{tabular}{@{}lccc@{}}
\toprule
\textbf{\begin{tabular}{@{}c@{}} Baseline \\ demand $z$ \end{tabular}} & \textbf{\begin{tabular}{@{}c@{}} $\theta^*$ \\ values\end{tabular}} & \textbf{\begin{tabular}{@{}c@{}} Final prices  \\ $p= U \theta^*$ \end{tabular}} \\ 
\midrule

$[60, 60, 60]$ & 
\begin{tabular}[c]{@{}c@{}}[4.78, 5.71, 6.11]\end{tabular} & 
\begin{tabular}[c]{@{}c@{}}[16.98, 32.81, 15.85]\end{tabular} \\

$[100, 150, 250]$ & 
\begin{tabular}[c]{@{}c@{}}[9.79, 15.40, 21.40]\end{tabular} & 
\begin{tabular}[c]{@{}c@{}}[25.19, 36.80, 31.19]\end{tabular} \\ 
\bottomrule
\end{tabular}
\end{table}

\textbf{Interpretability: }
Across both settings, the AFDLD model provides a clear, human–readable decomposition of prices into attribute–level components. The learned vector $\theta^*$ quantifies the monetary contribution of each observable feature (e.g., storage, camera, display), while the feature matrix $U$ specifies how these features combine to form each product. \newline
When $V=I$ (Setting 1), all attributes contribute equally and price differences arise solely from variations in baseline demand $z$, directly revealing which features are most valued by consumers. 
When $V$ is heterogeneous (Setting 2), diagonal weights encode the relative market importance of each attribute, allowing ADEPT to separate the effects of baseline demand from intrinsic feature significance. 
Thus, the decomposition $p_i=\langle U(i,\cdot),\theta^*\rangle$ offers an interpretable pricing logic—every price differential can be traced back to specific attributes and their learned valuations—unlike low–rank latent models where such attribution is opaque.

\section{ADEPT: A Novel Dynamic Pricing Algorithm} \label{sec:adept}
In this section, we present the ADEPT algorithm, which performs dynamic pricing directly in the attribute space. 
We describe how ADEPT initializes attribute-level parameters, perturbs them through one-point bandit feedback, and updates prices via simple clipping within feasible bounds. We highlight key differences between ADEPT and OPOK in initialization, gradient estimation, projection, and interpretability. Finally, we establish the theoretical regret bound for ADEPT, showing that it achieves a sublinear expected regret of $\tilde{O}(\sqrt{d}\,T^{3/4})$ under standard smoothness and noise assumptions, thereby confirming both efficiency and scalability.

\begin{algorithm}
\caption{\textsc{ADEPT}}
\label{alg:adept-compact}
\begin{algorithmic}[1]
\REQUIRE Feature matrix $U \in \mathbb{R}^{N \times d}$, baseline $\theta_{\text{base}}$, box $[\theta_{\min}, \theta_{\max}]$, step size $\eta > 0$, perturbation $\epsilon > 0$, horizon $T$
\STATE \textbf{Initialize:} $\theta_1 \leftarrow 0$; set lower and upper limits $\ell \leftarrow \theta_{\min} - \theta_{\text{base}}$, $u \leftarrow \theta_{\max} - \theta_{\text{base}}$
\FOR{$t = 1$ to $T$}
    \STATE Sample a random direction $\xi_t \sim \mathrm{Unif}(\mathbb{S}^{d-1})$
    \STATE Perturb parameters: $\tilde{\theta}_t \leftarrow \theta_t + \epsilon \xi_t$
    \STATE Construct prices: $p_t \leftarrow U(\theta_{\text{base}} + \tilde{\theta}_t)$
    \STATE Observe revenue $y_t = R_t(p_t)$
    \STATE Compute one-point gradient estimate:
    \[
        g_t \leftarrow -\frac{d\,y_t}{\epsilon}\,\xi_t
    \]
    \STATE Update and clip within box:
    \[
        \theta_{t+1} \leftarrow \text{clip}_{[\ell, u]}\big(\theta_t - \eta g_t\big)
    \]
\ENDFOR
\STATE \textbf{Output:} Final prices $p = U(\theta_{\text{base}} + \theta_{T+1})$
\end{algorithmic}
\end{algorithm}

\begin{algorithm}
\caption{OPOK (Online Pricing Optimization with Known Features)}
\small
\begin{algorithmic}[1]
\REQUIRE Step sizes $\eta,\delta,\alpha > 0$, feature matrix $U \in \mathbb{R}^{N \times d}$, initial price $p_0 \in S$
\ENSURE Prices $p_1,\ldots,p_T$ to maximize revenue
\STATE Set prices $p_0 \in S$, observe $q_0(p_0), R_0(p_0)$
\STATE Define $x_1 \gets U^\top p_0$
\FOR{$t=1,\ldots,T$}
  \STATE Draw $\xi_t \sim \mathrm{Unif}\{x \in \mathbb{R}^d : \|x\|_2=1\}$
  \STATE $\tilde{x}_t \gets x_t + \delta \xi_t$
  \STATE Set prices: $p_t \gets \textsc{FindPrice}(\tilde{x}_t, U, S, p_{t-1})$, observe $q_t(p_t), R_t(p_t)$
  \STATE $x_{t+1} \gets \textsc{Projection}(x_t - \eta R_t(p_t)\xi_t, \alpha, U, S)$
\ENDFOR
\end{algorithmic}
\end{algorithm}

\subsection{OPOK Vs. ADEPT}
\begin{itemize}
  \item \textbf{Initialization:}  
  OPOK initializes with a feasible price vector $p_0 \in S$ and maps it into latent space $x_1 = U^\top p_0$.  
  ADEPT initializes directly in attribute space ($\hat{\theta}_1=0$) with box constraints $[\theta_{\min},\theta_{\max}]$ around a baseline $\theta_{\text{base}}$.

  \item \textbf{Perturbation direction:}  
  OPOK samples $\xi_t \in \mathbb{R}^d$ and perturbs the latent representation $x_t$.  
  ADEPT samples $\xi_t \in \mathbb{R}^d$ and perturbs the attribute-level parameters $\theta_t$.

  \item \textbf{Price construction:}  
  OPOK uses $\textsc{FindPrice}(\tilde{x}_t,U,S,p_{t-1})$ to map the perturbed latent vector $\tilde{x}_t$ to a feasible price vector $p_t$, ensuring stability relative to past prices.  
  ADEPT directly constructs candidate prices as $p_t=U(\theta_{\text{base}}+\tilde{\theta}_t)$, bypassing a separate feasibility solver.

  \item \textbf{Revenue observation:}  
  Both algorithms post prices and observe demand/revenue. 

  \item \textbf{Gradient estimation:}  
  OPOK updates $x_{t+1}$ via estimator: $x_{t+1} = \textsc{Projection}(x_t - \eta R_t(p_t)\xi_t,\alpha,U,S)$.  
  ADEPT updates via estimator: $g_t = -\frac{d\,y_t}{\varepsilon}\,\xi_t$.

  \item \textbf{Projection/Constraints:}  
  OPOK projects the latent point back into the feasible span using $\textsc{Projection}(\cdot)$.  
  ADEPT clips attribute prices into $[\theta_{\min},\theta_{\max}]$, enforcing economically meaningful nonnegativity and boundedness.

  \item \textbf{Output:}  
  OPOK’s iterates $p_t$ are feasible prices at each step.  
  ADEPT outputs the final feasible price vector $p=U(\theta_{\text{base}}+\hat{\theta}_{T+1})$.
\end{itemize}

\paragraph{Summary.}
OPOK and ADEPT differ in optimization domain and objective formulation — latent vs. attribute space—rendering their regrets and learned parameters not directly comparable.

\begin{theorem}\label{th:main}[Regret of \textsc{ADEPT}]
Let $U\in\mathbb{R}^{N\times d}$ with $\|U\|_{op}\!\le\!B_U$, and let the attribute box be
$\mathcal{C}\subset\mathbb{R}^d$ with radius $r_\Theta:=\max_{\theta\in\mathcal{C}}\|\theta\|_2$.
Assume for all $t$: $A_t=U^\top V_t U\succeq0$ with $\|V_t\|_{op}\!\le\!B_V$, $\|z_t\|_2\!\le\!B_z$,
and bandit noise that is $\sigma^2$–sub-Gaussian. With $\eta\!\asymp\!T^{-1/2}$ and
$\delta\!\asymp\!T^{-1/4}$,
\[
\mathbb{E}\!\Big[\textstyle\sum_{t=1}^T\big(R_t(\theta^\star)-R_t(\hat\theta_t)\big)\Big]
\;\le\; K\,\sqrt{d}\,T^{3/4},
\]
where $\theta^\star\in\arg\max_{\theta\in\mathcal{C}}\sum_{t=1}^T R_t(\theta)$ and
\[
K \;:=\; c_0\!\left(B_U B_z r_\Theta \;+\; B_U^2 B_V r_\Theta^2 \;+\; \sigma\right),
\]
for a universal constant $c_0>0$ (log factors absorbed into $c_0$). \\ See Supplement for the complete proof.
\end{theorem}

\section{Experiments and Insights}
\subsection{Empirical Validation of Additive Feature Decomposition in Real-World Datasets}
We validate the additive feature decomposition (AFD) assumption, which is central to ADEPT's formulation, using two large-scale, real-world retail datasets namely \href{https://www.kaggle.com/datasets/frtgnn/dunnhumby-the-complete-journey} {Dunnhumby Complete Journey} and \href{https://www.kaggle.com/competitions/h-and-m-personalized-fashion-recommendations/data}{H\&M Personalized Fashion Recommendation}. Both datasets are of the form $(x_i,y_i)_{i=1}^n$, where $x_i\in\R^d$ is the attribute based product feature (as discussed in \Cref{sec:afd}) and $y_i$ is the price. We hasten to mention that, our goal in this section is to verify that the attributes based features are indeed useful as features, and not to validate our demand model. In order to achieve our goal, we show that the attribute based features have good power in predicting the prices via linear regression fit.

\textbf{Dunnhumby Complete Journey (\Cref{fig:dunnhumby-afd}.):} In this dataset, the $d$-attributes are \texttt{\nolinkurl{COMMODITY\_DESC}},\texttt{\nolinkurl{SUB\_COMMODITY\_DESC}},\texttt{\nolinkurl{MANUFACTURER}} which take values in one among $94$ distinct commodities, $736$ unique sub-commodities and $1527$ unique manufacturers respectively. Therefore $x_i$ is a 2357-dimensional vector. It consisted of 39021 distinct products. The linear regression fit achieves a strong $R^2$ (0.6-0.8) and the learned coefficients exhibit intuitive semantic structure — e.g., premium brands and specialized commodity categories contribute positively to price, while generic or high-volume subcategories contribute negatively. Importantly, this decomposition generalizes well to unseen data, indicating that unit prices in grocery retail indeed follow an additive structure over product metadata. This is further supported by interpretable SHAP analyses and bar plots of regression coefficients, which  confirm the additive contribution of each feature group to the final price.

\textbf{H\&M Personalized Fashion Recommendation (\Cref{fig:hm-afd}):} In  this dataset, the $d$ attributes are \texttt{\nolinkurl{product\_type\_name}}, \texttt{\nolinkurl{garment\_group\_name}}, \texttt{\nolinkurl{index\_name}}, and \texttt{\nolinkurl{section}} which take values in one among $131$ distinct types, $21$ groups, $10$ unique indices and $52$ sections respectively. Therefore $x_i$ is 214-dimensional vector. It consisted of 104547 distinct products. The linear regression fit reveals statistically significant coefficients across many features.

\begin{figure}[t]
    \centering
    \includegraphics[width=0.50\textwidth]{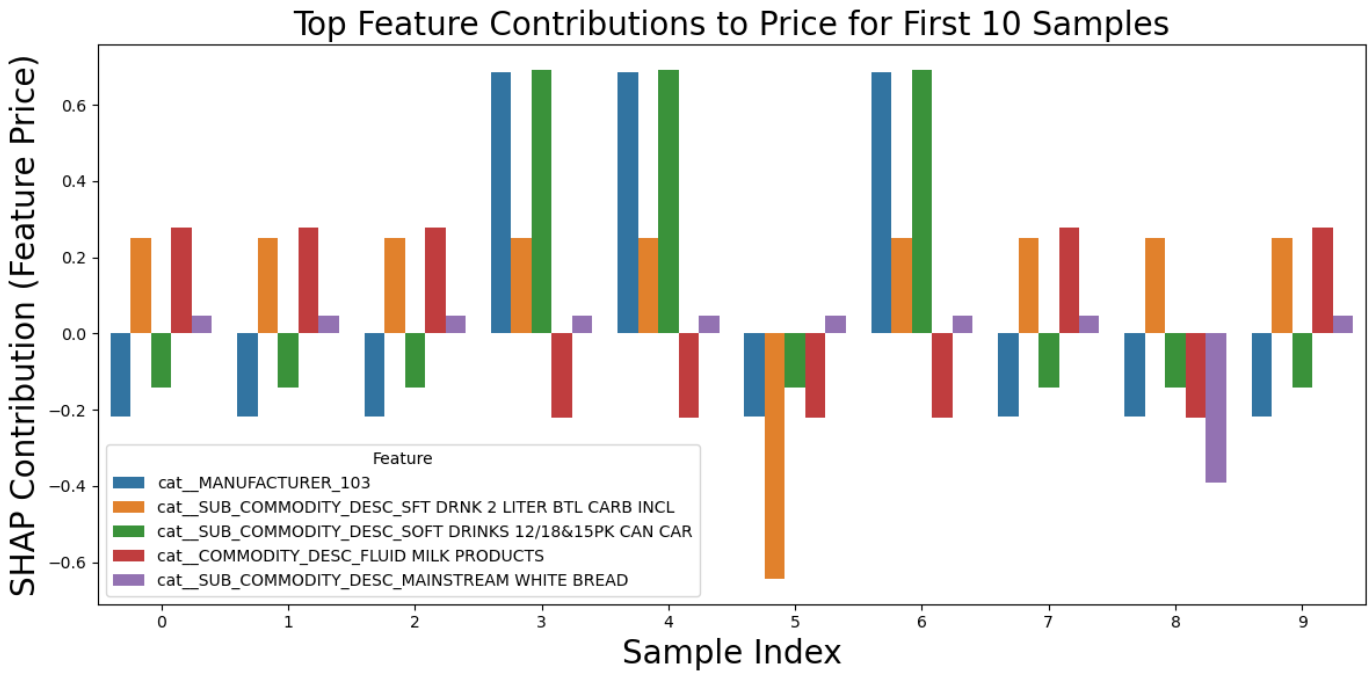}
    \caption{Additive feature contributions to unit prices in the Dunnhumby grocery dataset, obtained via linear regression coefficients.}
    \label{fig:dunnhumby-afd}
\end{figure}

\begin{figure}[t]
    \centering
    \includegraphics[width=0.45\textwidth]{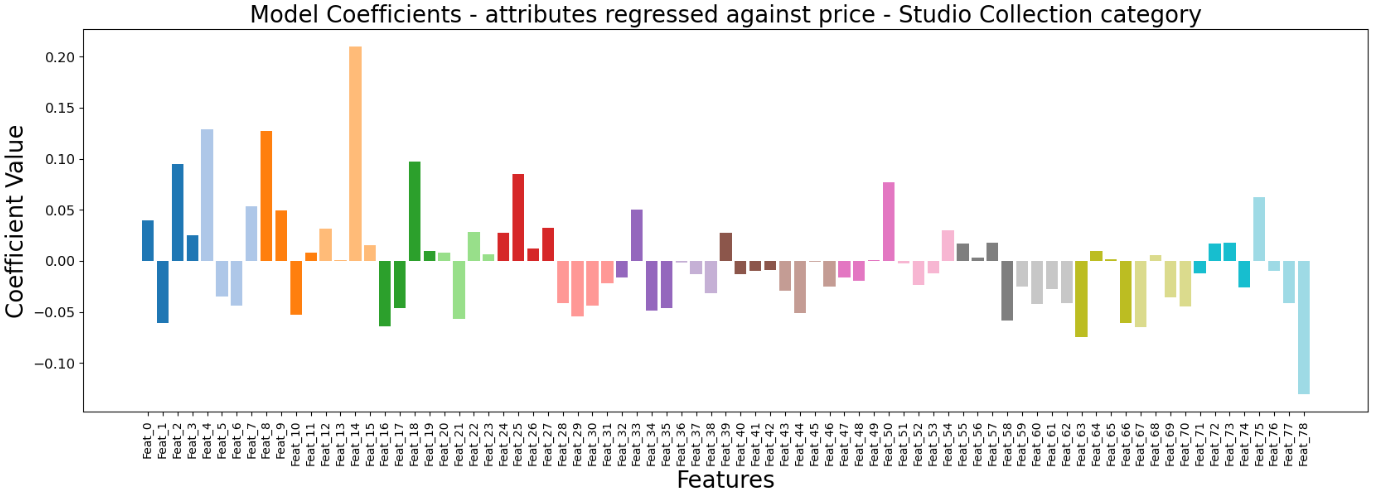}
    \caption{Additive feature contributions to prices in the H\&M fashion dataset, highlighting interpretable price drivers such as garment type and season.}
    \label{fig:hm-afd}
\end{figure}

\subsection{Comparing ADEPT with Baselines}
Having verified additive structure in real data, we next benchmark ADEPT against existing bandit algorithms.
In this section, we evaluate algorithms under the AFDLD model introduced in Section~\ref{sec:attrelasmodel}. The pseudocode for ADEPT is specified in \Cref{alg:adept-compact}. We compare ADEPT with two benchmark algorithms namely GDG \cite{flaxman2005bandit} and Explore-Exploit \cite{henaff2019explicit}. Specifically, we compare these algorithms under different market regimes.

\paragraph{Common Across Settings.}
All algorithms are tested under the following shared configuration:
\begin{itemize}
    \item Product-Feature matrix $U \in \mathbb{R}^{N \times d}$ with $N=60$ products and $d=6$ attributes, respecting additive decomposition. We divide the products into $B = 6$ disjoint blocks $\{G_b\}_{b=1}^{6}$, each containing $|G_b| = 10$ products. Each block $G_b$ is assigned a small set of active attributes $S_b \subseteq \{1,\dots,6\}$, ensuring that neighboring blocks share one or more attributes to introduce controlled overlap ($S_1 = \{1,2,3\}, S_2 = \{2,3,4\}, \dots, S_6 = \{6,1,2\}$).  For every product $i \in G_b$ and attribute $j \in S_b$, we sample $U(i,j) \sim \mathrm{Bernoulli}(p_b)$ with $p_b = 0.5$, and set $U(i,j) = 0$ for $j \notin S_b$. We then cap the number of active attributes per product to a small range (one or two active attributes per row). If a row has no active attributes, we activate one at random within $S_b$; if it exceeds the cap, we randomly retain only two active attributes. 
    \item Time horizon $T=50{,}000$ rounds.
    \item Gaussian observation noise with variance $\sigma^2 = 0.5$.
    \item Evaluation metric: cumulative regret relative to the best fixed price in $\mathcal{B}_P(R)$. We fix $\mathcal{B}_P(R)$ = $\{\,p \in \mathbb{R}^{N} : \|p - p_{0}\|_{2} \le 5\,\}$ for both GDG and Explore–Exploit (EE) algorithms. The baseline price vector $p_{0}$ and the attribute-level base parameter $\theta_{0}$ are related through $p_{0} = U\theta_{0}$. We set $\theta_{\min} = \theta_{0} - r$ and $\theta_{\max} = \theta_{0} + r$, ensuring that the attribute parameters explore the same range as the price space defined by $B_{P}(R)$ ($r=5$). For ADEPT, cumulative regret is measured relative to the best attribute level prices ($\theta^*$).
\end{itemize}

\paragraph{Algorithm-Specific Parameters.}
\begin{itemize}
    \item \textbf{GDG (Gradient Descent with Bandit Feedback) :} Operates in the orthonormal span $O$ obtained 
    from QR decomposition $U = OP$, with one-point random-direction smoothing and projection in $O$-space. See Supplement for pseudocode.
    \item \textbf{Explore--Exploit :} Alternates between uniform exploration in $\mathcal{B}_P(R)$, 
    ridge regression fit of a quadratic surrogate, and trust-region exploitation. See Supplement for pseudocode.
    \item \textbf{ADEPT:} Updates attribute prices directly in feature space via one-point gradient estimator with feasibility enforced by box constraints on $\hat{\theta}_t$.
\end{itemize}

\noindent\paragraph{Demand Regimes}
We consider four demand regimes. Each regime differs in the temporal evolution of the baseline demand $z_t$ and the cross-elasticity operator $V_t$ in the demand model specified in \Cref{eq:lddm}, while all other parameters remain fixed.

\begin{enumerate}
    \item \textbf{Stationary demand (S1):} 
    Both baseline demand and cross-elasticities remain fixed, i.e.\ $z_t = z$, $V_t = V$ for all $t$.
    This corresponds to a stable market with no temporal variation in consumer preferences or substitution effects. In \textbf{S1}, ADEPT maintains nearly flat regret, GDG grows slowly due to variance in one-point updates, and Explore-Exploit accumulates error from surrogate bias.

    \item \textbf{Structural shocks (S2):} 
    At $t = T/3$ and $t = 2T/3$, the baseline and elasticity parameters $(z_t, V_t)$ are generated again with the same assumptions independently in each new phase. This simulates sudden market shocks such as competitor entry or seasonal change. In \textbf{S2}, ADEPT adapts rapidly after each shock, while GDG shows pronounced regret spikes and Explore-Exploit requires many rounds to re-learn post-shock surrogates.

    \item \textbf{Drifting demand (S3):} 
    Parameters evolve gradually according to stochastic drift:
    \[
    z_{t+1} = z_t + w_t, \qquad V_{t+1} = \Pi_{\mathcal{V}}(V_t + W_t),
    \]
    where $w_t \sim \mathcal{N}(0, I)$ and $W_t \sim \mathcal{N}(0, 0.1 I)$, and $\Pi_{\mathcal{V}}$ projects into the positive-definite cone. This captures smooth preference shifts over time. 

    \item \textbf{Misspecified demand (S4):} 
    Instead of low-rank $V_t$, the true demand operator is taken to be full-rank with heterogeneous eigenvalues. 
    This explicitly violates the low-rank assumption and tests robustness of the algorithms under model misspecifications. In \textbf{S4}, ADEPT remains competitive despite misspecification, while GDG and 
    Explore-Exploit incur higher regret due to structural mismatch with the full-rank $M_t$.
\end{enumerate}

\FloatBarrier
\begin{figure*}[!t]
  \centering
  \includegraphics[width=0.85\textwidth]{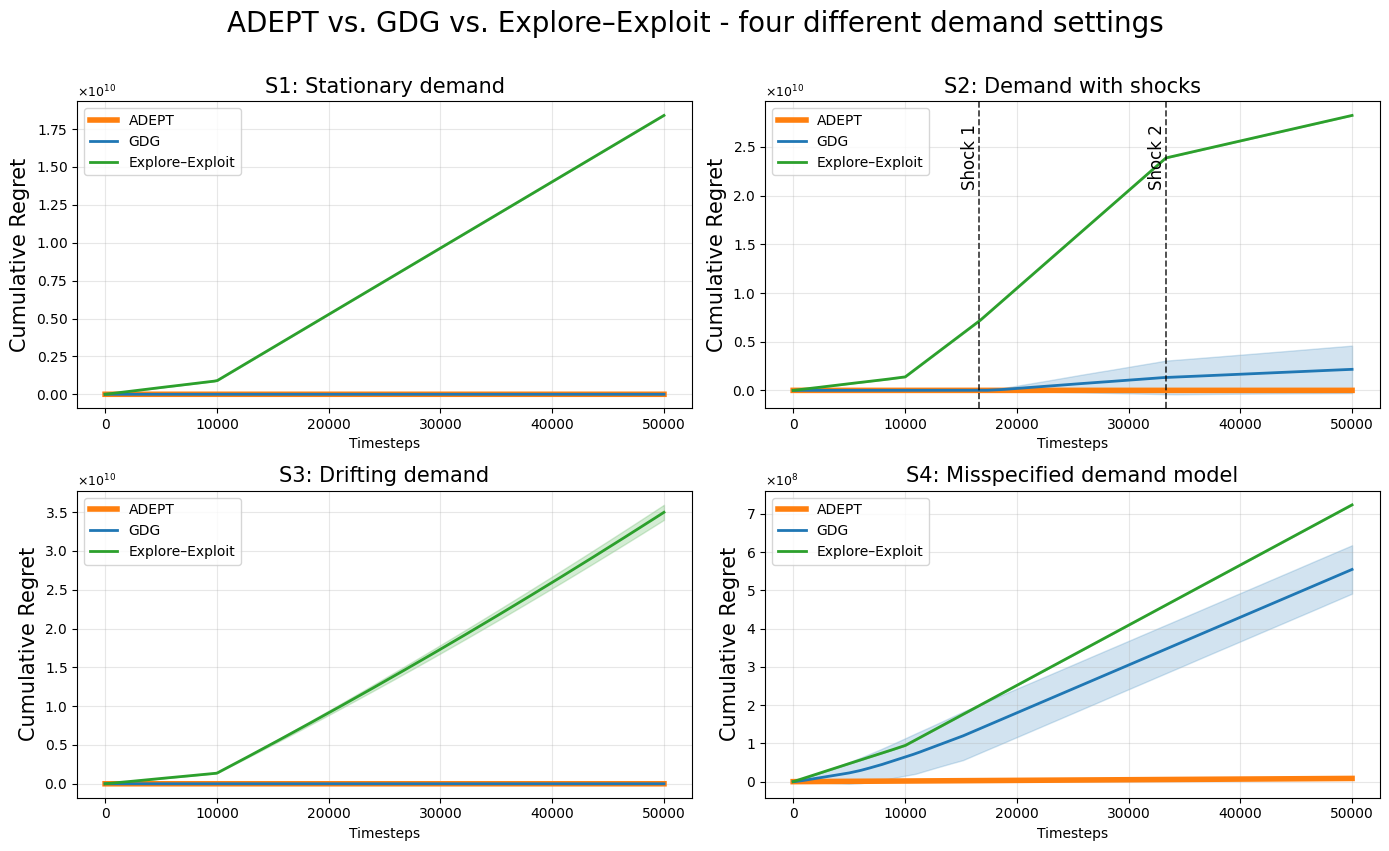}
  \caption{\textbf{Cumulative regret (mean $\pm$ s.d.) over $T{=}50{,}000$ rounds.} Top-left: S1 Stationary; top-right: S2 Shocks (vertical dashed lines mark change points); bottom-left: S3 Drifting; bottom-right: S4 Misspecified. All methods share the same price ball and comparator.}
  \label{fig:exp-four-settings}
\end{figure*}

\noindent In all experiments, ADEPT’s attribute-space updates with a structured $U$ deliver lower cumulative regret and stronger adaptability. Finally, to verify consistency between empirical and theoretical performance, we estimate ADEPT’s regret-growth rate through tail-slope analysis in the next subsection.

\balance
\subsection{Empirical rate verification via tail–slope fitting}
\label{emp_val_regretslope}
We verify that \textsc{ADEPT} attains the rate in \Cref{th:main}.  
If $R(t)\!\asymp\! t^{\alpha}$, then $\log R(t)=a+\alpha\log t$ is linear in $\log t$ with slope $\alpha$.  
To estimate $\alpha$, we fit a \emph{tail secant} over the last $\rho=0.5$ of the horizon ($T=50{,}000$):
\[
\widehat{\alpha}_{\text{tail}}
= \frac{\log R(T)-\log R(t_0)}{\log T-\log t_0},
\quad t_0=\lfloor(1-\rho)T\rfloor.
\]

\paragraph{Results (Fig.~\ref{fig:regret-tail-slope})}
The fitted slope,
$\widehat{\alpha}_{\text{tail}}=0.742$,
closely matches the theoretical $t^{3/4}$ rate. 
Early curvature arises from pre–asymptotic effects; the tail is nearly linear, confirming that for $(N{=}60,\,d{=}6,\,T{=}50{,}000)$, \textsc{ADEPT} empirically achieves $\tilde{\mathcal{O}}(T^{3/4})$ regret. A tail fraction $\rho=0.5$ balances bias and variance. 

\begin{figure}[H]
  \centering
  \includegraphics[width=0.47\textwidth]{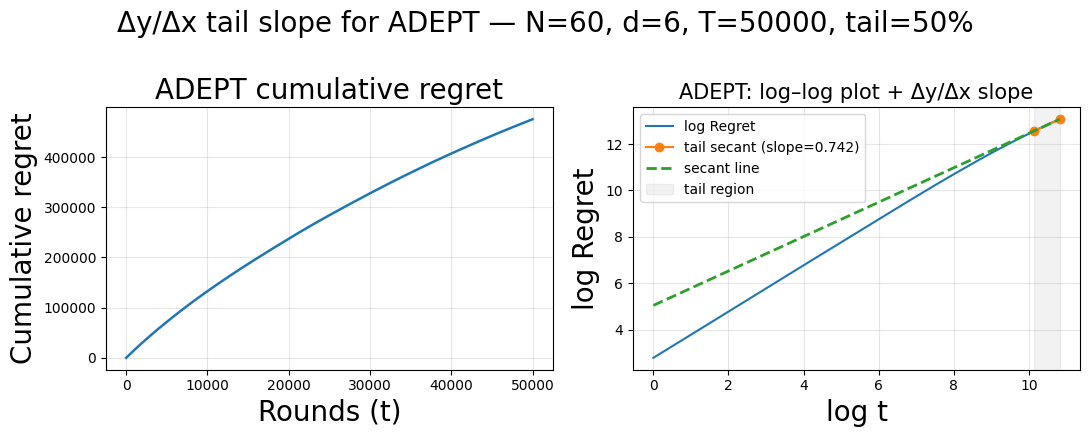}
  \caption{\textbf{Tail–slope analysis for \textsc{ADEPT}.} Left: cumulative regret $R(t)$.
  Right: log–log plot of $R(t)$ with the tail window (gray) and secant fit; estimated slope
  $\widehat{\alpha}_{\text{tail}}=0.742\approx 3/4$.}
  \label{fig:regret-tail-slope}
\end{figure}

\subsection{Scalability and Runtime Evidence}
We measured per-epoch runtime for all baselines under identical conditions. ADEPT is the fastest, with runtimes as shown in Table \ref{tab:runtimes_table}.

\begin{table}[H]
\centering
\caption{Average time per epoch (in seconds) for ADEPT, OPOK, GDG, and EE across different $(N, d)$ settings.}
\label{tab:runtimes_table}
\begin{tabular}{lcccc}
\hline
\textbf{Setting} & \textbf{ADEPT} & OPOK & GDG & EE \\
\hline
$N = 60,\ d = 6$   & \textbf{0.000041} & 0.000086 & 0.000054 & 0.000075 \\
$N = 100,\ d = 10$ & \textbf{0.000063} & 0.000092 & 0.000072 & 0.000078 \\
$N = 1000,\ d = 100$ & \textbf{0.000059} & 0.000075 & 0.000065 & 0.000559 \\
\hline
\end{tabular}
\end{table}

\section{Conclusion}
In this paper, we considered dynamic pricing in markets with similar products in which the product attributes affect the demand and pricing. For this setting, we proposed the \textbf{Additive Feature Decomposition based Low-Dimensional Demand (AFDLD)} model, which expresses prices as additive functions of observable attributes while capturing substitution effects among related products. Building on this, \textbf{ADEPT}—a projection- and gradient-free online learner — optimizes directly in attribute space and achieves a sublinear regret of $\Tilde{O}(\sqrt{d}\,T^{3/4})$. 
Experiments across stationary, shocked, drifting, and misspecified regimes show that ADEPT attains low regret, rapid adaptation, and interpretable attribute-level price explanations. These results establish that interpretability and efficiency can coexist in dynamic pricing. \newline
Future work will extend AFDLD to nonlinear and personalized settings, evolving attribute spaces, and multi-agent markets, advancing toward autonomous and explainable pricing systems.




\bibliographystyle{ACM-Reference-Format} 
\bibliography{sample}

@inproceedings{cohen2016contextual,
  title={Contextual bandits with latent confounders: An NMF approach},
  author={Cohen, Alon and Hazan, Elad and Koren, Tomer},
  booktitle={Proceedings of the 33rd International Conference on Machine Learning},
  year={2016}
}

@inproceedings{sen2016contextual,
  title={Contextual bandits with latent confounders: Model and algorithm},
  author={Sen, Rahul and Mathew, Bipul and Bhalgat, Yash and Jain, Prateek and Varma, Manik},
  booktitle={Proceedings of the 19th International Conference on Artificial Intelligence and Statistics},
  year={2016}
}

@inproceedings{ban2021personalized,
  title={Personalized pricing with nonparametric demand learning},
  author={Ban, Gah-Yi and Keskin, Narges and Liu, Zhaoxu},
  booktitle={Proceedings of the 38th International Conference on Machine Learning},
  year={2021},
  pages={517--527}
}

@article{keskin2020dynamic,
  title={Dynamic pricing with demand learning and reference effects},
  author={Keskin, Narges and Zeevi, Assaf},
  journal={Operations Research},
  volume={68},
  number={5},
  pages={1381--1405},
  year={2020},
  publisher={INFORMS}
}

@inproceedings{qi2021meta,
  title={Meta-learning for contextual bandit exploration},
  author={Qi, Yutian and Dong, Weinan and Shen, Kaixiang and Tang, Zhi and Zha, Hongyuan},
  booktitle={Advances in Neural Information Processing Systems},
  year={2021}
}

@inproceedings{zhu2021improving,
  title={Improving demand forecasting in e-commerce with attention-based low-rank embeddings},
  author={Zhu, Xiaohan and Huang, Yifan and Lin, Bo and Wu, Yichong},
  booktitle={Proceedings of the 27th ACM SIGKDD Conference on Knowledge Discovery \& Data Mining},
  year={2021},
  pages={3780--3790}
}

@inproceedings{wang2021low,
  title={Low-rank bandits with side observations},
  author={Wang, Zhen and Liu, Haoran and Li, Lihong and Liu, Qiang},
  booktitle={Proceedings of the 38th International Conference on Machine Learning},
  year={2021},
  pages={10766--10776}
}

@inproceedings{chen2021interpretable,
  title={Interpretable learning for travel demand and pricing with additive models},
  author={Chen, Xinyi and Zhang, Jun and Li, Jie},
  booktitle={Proceedings of the 27th ACM SIGKDD International Conference on Knowledge Discovery \& Data Mining},
  year={2021},
  pages={1072--1080}
}

@inproceedings{tang2022personalized,
  title={Personalized and interpretable pricing for fashion e-commerce using multimodal decomposition},
  author={Tang, Meng and Lu, Yile and Huang, Xiaoyi and Sun, Meng and Wang, Yu},
  booktitle={Proceedings of the 2022 Conference on Empirical Methods in Natural Language Processing},
  year={2022}
}

@inproceedings{gupta2020pricing,
  title={Pricing airline ancillaries: A discrete choice modeling approach},
  author={Gupta, Ananya and Chintagunta, Pradeep and Jain, Dipak},
  booktitle={Proceedings of the 2020 INFORMS Annual Meeting},
  year={2020}
}

@article{aleksandrov2023modular,
  title={Modular pricing of digital subscription bundles},
  author={Aleksandrov, Bozhidar and Leontjeva, Anna and Hemphill, Tasha and Paquet, Ulrich},
  journal={Journal of Revenue and Pricing Management},
  volume={22},
  number={1},
  pages={33--48},
  year={2023},
  publisher={Springer}
}

@article{mueller2019low,
  title={Low-rank bandit methods for high-dimensional dynamic pricing},
  author={Mueller, Jonas W and Syrgkanis, Vasilis and Taddy, Matt},
  journal={Advances in Neural Information Processing Systems},
  volume={32},
  year={2019}
}

@book{ricci2015recommender,
  title={Recommender systems handbook},
  author={Ricci, Francesco and Rokach, Lior and Shapira, Bracha},
  year={2015},
  publisher={Springer}
}

@article{koren2009matrix,
  title={Matrix factorization techniques for recommender systems},
  author={Koren, Yehuda and Bell, Robert and Volinsky, Chris},
  journal={Computer},
  volume={42},
  number={8},
  pages={30--37},
  year={2009}
}

@article{jannach2016recommendation,
  title={Recommendation systems for e-commerce},
  author={Jannach, Dietmar and Adomavicius, Gediminas},
  journal={User Modeling and Personalization},
  year={2016}
}

@article{taddy2019bayesian,
  title={A nonparametric Bayesian analysis of heterogeneous treatment effects in digital marketing},
  author={Taddy, Matt and Gardner, Matt and Chen, Liyuan and Draper, David},
  journal={Journal of Marketing Research},
  year={2019}
}

@inproceedings{chen2016empirical,
  title={An empirical analysis of algorithmic pricing on Amazon Marketplace},
  author={Chen, Liyuan and Mislove, Alan and Wilson, Christo},
  booktitle={Proceedings of the 25th International Conference on World Wide Web},
  pages={1339--1349},
  year={2016}
}

@article{elmaghraby2003dynamic,
  title={Dynamic pricing in the presence of inventory considerations: Research overview},
  author={Elmaghraby, Wedad and Keskinocak, Pinar},
  journal={Production and Operations Management},
  volume={12},
  number={4},
  pages={535--552},
  year={2003}
}

@inproceedings{chen2016dynamic,
  title={Dynamic pricing with demand learning and reference effects},
  author={Chen, Xi and Gallego, Guillermo and Sinha, Anant},
  booktitle={Advances in Neural Information Processing Systems (NeurIPS)},
  year={2016}
}

@inproceedings{kawale2015price,
  title={Price Optimization for Revenue Management using Reinforcement Learning},
  author={Kawale, Jaya and Bansal, Trapit and Chawla, Sanjay and Sharma, Aditya and Kveton, Branislav and Theocharous, Georgios and Wilkinson, Dennis},
  booktitle={Proceedings of the 24th ACM International on Conference on Information and Knowledge Management (CIKM)},
  year={2015}
}

@article{chen2019personalized,
  title={Personalized dynamic pricing with machine learning},
  author={Chen, Lijie and Gallego, Guillermo},
  journal={Handbook of Pricing Analytics},
  pages={1--29},
  year={2019},
  publisher={Springer}
}

@article{berry1995automobile,
  title={Automobile prices in market equilibrium},
  author={Berry, Steven and Levinsohn, James and Pakes, Ariel},
  journal={Econometrica},
  volume={63},
  number={4},
  pages={841--890},
  year={1995},
  publisher={Wiley Online Library}
}

@article{hosken2008retail,
  title={Retail price drivers and retail pricing strategies},
  author={Hosken, Daniel and Reiffen, David and Varney, Andrew and Taylor, Luke},
  journal={Journal of Economics \& Management Strategy},
  volume={17},
  number={1},
  pages={147--187},
  year={2008},
  publisher={Wiley Online Library}
}

@article{andrews2016estimating,
  title={Estimating cross-price elasticities from aggregate retail data},
  author={Andrews, Michelle and Currim, Imran S},
  journal={Journal of Marketing Research},
  volume={53},
  number={2},
  pages={187--198},
  year={2016},
  publisher={SAGE Publications}
}

@book{chung1997spectral,
  title={Spectral Graph Theory},
  author={Chung, Fan R. K.},
  volume={92},
  year={1997},
  publisher={American Mathematical Society},
  series={CBMS Regional Conference Series in Mathematics}
}

@book{hornjohnson2012matrix,
  title={Matrix Analysis},
  author={Horn, Roger A. and Johnson, Charles R.},
  edition={2nd},
  year={2012},
  publisher={Cambridge University Press}
}

@inproceedings{flaxman2005bandit,
  title={Online convex optimization in the bandit setting: gradient descent without a gradient},
  author={Flaxman, Abraham D. and Kalai, Adam Tauman and McMahan, H. Brendan},
  booktitle={Proceedings of the Sixteenth Annual ACM-SIAM Symposium on Discrete Algorithms},
  pages={385--394},
  year={2005},
  organization={SIAM}
}

@book{varga2004gershgorin,
  title        = {Gershgorin and His Circles},
  author       = {Varga, Richard S.},
  year         = {2004},
  publisher    = {Springer-Verlag},
  address      = {Berlin, Heidelberg}
}

@article{bubeck2012survey,
  title={Regret analysis of stochastic and nonstochastic multi-armed bandit problems},
  author={Bubeck, S{\'e}bastien and Cesa-Bianchi, Nicol{\`o}},
  journal={Foundations and Trends in Machine Learning},
  volume={5},
  number={1},
  pages={1--122},
  year={2012},
  publisher={Now Publishers Inc.}
}

@misc{ZhangABPSSRN,
  author       = {Mengzhenyu Zhang and Christopher Ryan and Wei Sun},
  title        = {Attribute-based Pricing: A nested-logit approach},
  howpublished = {SSRN Electronic Journal},
  year         = {2022},
  doi          = {10.2139/ssrn.4258247},
  url          = {https://papers.ssrn.com/sol3/papers.cfm?abstract_id=4258247}
}

@misc{SalesforceABPHelp,
  author       = {{Salesforce}},
  title        = {Attribute-Based Pricing},
  howpublished = {\url{https://help.salesforce.com/s/articleView?id=ind.comms_attribute_based_pricing.htm&type=5}},
  note         = {Salesforce Help, Communications Cloud. Accessed: 2025-10-01},
  year         = {2025}
}

@article{henaff2019explicit,
  title={Explicit explore-exploit algorithms in continuous state spaces},
  author={Henaff, Mikael},
  journal={Advances in Neural Information Processing Systems},
  volume={32},
  year={2019}
}

\section{Supplementary Material}
This supplementary section provides detailed derivations, proofs, and additional algorithmic information supporting the main paper. 

\begin{itemize}
    \item \textbf{Section 1} formalizes the convexity discussion of the AFDLD objective and establishes positive (semi-)definiteness of key matrices through the Laplacian interpretation.
    \item \textbf{Section 2} presents the proof of Theorem~1 (Regret of ADEPT).
    \item \textbf{Section 3} lists pseudocodes for baseline algorithms (GDG and Explore--Exploit) used in the main experiments.
\end{itemize}

\bigskip
\subsection{Convexity and Laplacian Structure (Section A of Main Paper)}
\subsection{Setup}
Let $U\in\mathbb{R}^{N\times d}$ be the product-feature matrix and $V\in\mathbb{R}^{d\times d}$ symmetric.  
We define $V$-inner product as $\langle u_i,u_j\rangle_V := u_i^\top V {u_{{add}_j}}$ and $p=\Uadd\theta$ to denote the price vector.  

The demand model is
\[
q \;=\; \Uadd z \;-\; (M_1{+}M_2{+}M_3)\,p \;+\; \varepsilon,
\]
where the matrices are given by
\[
\begin{aligned}
M_1(i,i) &= \alpha_{ii}\,\langle u_i,u_i\rangle_V,\\
M_2(i,j) &= 
\begin{cases}
-\,\alpha_{ij}\,\langle u_i,u_j\rangle_V,& i\neq j,\\[4pt]
0,& i=j,
\end{cases}\\
M_3(i,i) &= \sum_{j}\alpha_{ij}\,\langle u_i,u_j\rangle_V.
\end{aligned}
\]

where i and j denote the attribute indices. The (noise-free) revenue function is
\[
R(\theta) \;=\; \theta^\top u^\top u\,z \;-\; \theta^\top u^\top(M_1{+}M_2{+}M_3)u\,\theta.
\]

\subsection{Assumptions}
We make the following modeling assumptions:
\begin{enumerate}\itemsep4pt
    \item[(A1)] $\alpha_{ij}\ge 0$ for all $i,j$ (nonnegative elasticities).
    \item[(A2)] $V\succeq 0$ and $\langle u_i,u_j\rangle_V \ge 0$ for $i\neq j$ (nonnegative cross-similarities).
\end{enumerate}

\paragraph{On Assumptions (A1)--(A2).} 
Assumption (A1) requires $\alpha_{ij} \geq 0$ for all $i,j$, which is consistent with economic intuition that price increases cannot stimulate demand and ensures that the Laplacian weights in $M_2+M_3$ are nonnegative. Assumption (A2) stipulates that $V \succeq 0$ and $\langle u_i,u_j\rangle_V \geq 0$ for $i \neq j$, guaranteeing that similarity measures are nonnegative and enabling the Laplacian interpretation of cross-elasticities. Importantly, these are not restrictive conditions: as shown in Section~5.1 of \cite{mueller2019low}, any non-orthogonal feature matrix $U$ can be orthogonalized and the corresponding $V$ reparameterized into a positive semidefinite operator without loss of generality. Thus, (A1)--(A2) represent canonical normalizations that preserve the expressiveness of the model while ensuring convexity and interpretability.

\subsection{Laplacian Structure of $M_2+M_3$}

To analyze $M_2+M_3$, we first symmetrize the interaction terms since only the symmetric part contributes
to a quadratic form. We define
\[
w_{ij} := \tfrac{\alpha_{ij}+\alpha_{ji}}{2}\,\langle u_i,u_j\rangle_V, \qquad i\neq j, \quad w_{ii}:=0.
\]
By assumptions (A1)--(A2), $w_{ij}\ge 0$ for all $i\neq j$.
Collect these weights into a symmetric matrix $W=[w_{ij}]\in\mathbb{R}^{N\times N}$.
Then $W\mathbf{1}$ denotes the vector of weighted degrees, i.e.
\[
(W\mathbf{1})_i = \sum_{j=1}^N w_{ij},
\]
the total interaction weight incident to node $i$ in the graph defined by the items $\{1,\dots,N\}$.

Now let $D:=\mathrm{Diag}(W\mathbf{1})$.
That is, $D$ is a diagonal matrix where $D_{ii}$ equals the weighted degree of node $i$.
It follows directly from the definitions of $M_2$ and $M_3$ that
\[
M_2+M_3 = D - W =: L,
\]
where $L$ is the weighted \emph{graph Laplacian} associated with the similarity graph $(w_{ij})$.

\paragraph{Quadratic form of the Laplacian.}
For any $x\in\mathbb{R}^N$,
\[
x^\top L x = \sum_{i} D_{ii} x_i^2 - \sum_{i\neq j} w_{ij} x_i x_j
= \frac{1}{2}\sum_{i,j=1}^N w_{ij}(x_i - x_j)^2 \;\ge 0.
\]
This expression makes the positive semidefiniteness immediate: each term $(x_i-x_j)^2$ is nonnegative,
and the weights $w_{ij}$ are nonnegative.
It also provides a clear interpretation:
\begin{itemize}
    \item If two items $i$ and $j$ are strongly connected (large $w_{ij}$),
    then the quadratic penalty discourages large differences between $x_i$ and $x_j$.
    \item If items are disconnected (no edges), then their coordinates contribute independently.
\end{itemize}

\paragraph{Nullspace characterization.}
The nullspace of $L$ is well known in spectral graph theory \cite{chung1997spectral}:
\[
\ker(L) = \mathrm{span}\{\mathbf{1}_C : C \text{ is a connected component of the graph }(w_{ij})\},
\]
where $\mathbf{1}_C$ denotes the indicator vector of component $C$.
In words: $Lx=0$ precisely when $x$ is constant on each connected component of the similarity graph.
This characterization plays a central role in establishing conditions under which
$M=M_1+M_2+M_3$ is strictly positive definite (Theorem~\ref{thm:pd}).

\begin{figure}[t]
\centering
\begin{tikzpicture}[
    scale=1.0,
    every node/.style={font=\small},
    node/.style={circle,draw,thick,minimum size=9mm,fill=blue!10},
    wlabel/.style={fill=white,inner sep=1pt}
]

\node[node] (n1) at (3,0) {1};
\node[node] (n2) at (1.5,2.2) {2};
\node[node] (n3) at (0,0) {3};
\node[node] (n4) at (1.5,-2.2) {4};

\draw[thick] (n1) -- (n2) node[midway,sloped,above,wlabel] {$0.8$};
\draw[thick] (n2) -- (n3) node[midway,sloped,above,wlabel] {$0.6$};
\draw[thick] (n3) -- (n4) node[midway,sloped,below,wlabel] {$0.7$};
\draw[thick] (n4) -- (n1) node[midway,sloped,below,wlabel] {$0.5$};
\draw[thick] (n1) -- (n3) node[midway,above,wlabel] {$0.4$};

\node[align=center,below=6pt of n4,inner sep=0pt]
{\footnotesize Example similarity graph for $M_2{+}M_3$ \quad
($L=D-W$, with $W_{ij}=w_{ij}$,\; $D=\mathrm{Diag}(W\mathbf{1})$)};
\end{tikzpicture}
\caption{Weighted similarity graph with edge weights $w_{ij}$.
The Laplacian is $L=D-W$, where $W\mathbf{1}$ denotes the
\emph{degree vector} whose $i$-th entry is $\sum_{j} w_{ij}$.}
\label{fig:laplacian-example}
\end{figure}

\begin{proof}
This follows directly from the definition of Laplacians in spectral graph theory 
\cite[Ch.~1]{chung1997spectral}.
\end{proof}

\subsection{Diagonal Matrix $M_1$}

\begin{lemma}[$M_1$ is PSD]
$M_1=\mathrm{Diag}(\alpha_{ii}\,\langle u_i,u_i\rangle_V)\succeq 0$ under (A1)--(A2).
\end{lemma}

\begin{proof}
Each diagonal entry is nonnegative by assumption, hence $M_1$ is positive semidefinite.
\end{proof}

\subsection{PSD and PD of $M_1+M_2+M_3$}

\begin{theorem}[PSD/PD characterization]\label{thm:pd}
Let $M=M_1+M_2+M_3$. Then:
\begin{enumerate}
    \item $M\succeq 0$ always, as the sum of PSD matrices.
    \item $M\succ 0$ if and only if $M_1$ is positive definite on $\ker(M_2+M_3)$.  
    Equivalently, in every connected component of the graph $(w_{ij})$ there exists some $i$ such that $\alpha_{ii}\,\langle u_i,u_i\rangle_V>0$.
\end{enumerate}
\end{theorem}

\begin{proof}
By Lemmas above, $M_1\succeq 0$ and $M_2+M_3\succeq 0$.  
The nullspace of $M_2+M_3$ is spanned by indicator vectors of connected components \cite{chung1997spectral}.  
If $x\notin\ker(M_2+M_3)$, then $x^\top (M_2+M_3)x>0$ and hence $x^\top Mx>0$.  
If $x\in\ker(M_2+M_3)\setminus\{0\}$, then $x^\top Mx=x^\top M_1 x$.  
This is strictly positive provided at least one diagonal in $M_1$ is positive within each connected component. 
\end{proof}

\subsection{Diagonal Dominance Criterion}

\begin{corollary}[Strict diagonal dominance $\Rightarrow$ PD]\label{cor:gersh}
If
\[
M_{ii} > \sum_{j\neq i}|M_{ij}| \qquad \text{for all } i,
\]
with $M_{ii}>0$, then $M\succ 0$ by Gershgorin’s circle theorem
\cite[Thm.~6.2.27]{hornjohnson2012matrix}, \cite{varga2004gershgorin}.
In our setting,
\[
M_{ii}=\alpha_{ii}\,\langle u_i,u_i\rangle_V + \sum_{j\neq i}w_{ij}, \quad
M_{ij}=-w_{ij}\ (i\neq j).
\]
Thus strict diagonal dominance reduces to requiring $\alpha_{ii}\,\langle u_i,u_i\rangle_V>0$ for every $i$.
\end{corollary}

\subsection{Convexity of the Negative Revenue}

\begin{corollary}[Convexity]
Define the cost
\[
\mathcal{C}(\theta) = \theta^\top U^\top M U\,\theta - \theta^\top U^\top U\,z.
\]
Then
\[
\nabla^2 \mathcal{C}(\theta) = 2\,U^\top M U \succeq 0,
\]
so $\mathcal{C}$ is convex in $\theta$.  
If $M\succ 0$ and $U$ has full column rank, then $\mathcal{C}$ is strictly convex.
\end{corollary}

\begin{proof}
PSD is preserved under congruence transformations (Sylvester’s law of inertia, \cite{hornjohnson2012matrix}).
\end{proof}

\subsection{Remarks}
\begin{itemize}
    \item The Laplacian interpretation connects the elasticity model to spectral graph theory \cite{chung1997spectral}.  
    \item Convexity of the cost is the critical condition enabling online convex optimization methods such as the bandit gradient descent of \cite{flaxman2005bandit}.  
    \item This structure parallels the convex quadratic objectives assumed in low-rank bandit pricing models \cite{mueller2019low}.
\end{itemize}

\bigskip
\subsection{Regret Analysis of ADEPT (Theorem~1)}
We now provide the complete derivation of the regret bound for ADEPT under the assumptions stated in the main paper. Following the bandit convex optimization framework of Flaxman et al. \cite{flaxman2005bandit}. 

\paragraph{Theorem 1 (Regret of ADEPT).} 
Let $U \in \mathbb{R}^{N \times d}$ with $\|U\| \leq B_U$, and let the attribute box be $C \subset \mathbb{R}^d$ with radius $r_\Theta := \max_{\theta \in C} \|\theta\|_2$. Assume for all $t$: (i) $A_t = U^\top V_t U \succeq 0$ with $\|V_t\| \leq B_V$, 
(ii) $\|z_t\|_2 \leq B_z$, and (iii) bandit noise is $\sigma^2$–sub-Gaussian. 
Then ADEPT, using the two–point gradient estimator
\[
\hat g_t \;=\; \frac{d}{2\varepsilon} \big(c_t(\theta_t + \varepsilon u_t) - c_t(\theta_t - \varepsilon u_t)\big) u_t,
\quad u_t \sim \text{Unif}(\mathbb{S}^{d-1}),
\]
with step size $\eta \asymp T^{-1/2}$ and smoothing radius $\varepsilon \asymp T^{-1/4}$, achieves expected cumulative regret
\[
\mathbb{E}\!\left[\sum_{t=1}^T \big(R_t(\theta^\star) - R_t(\theta_t)\big)\right]
\;\leq\; K \sqrt{d}\,T^{3/4},
\]
where $\theta^\star \in \arg\max_{\theta \in C} \sum_{t=1}^T R_t(\theta)$ and
\[
K \;=\; c_0\Big(B_U B_z r_\Theta \;+\; B_U^2 B_V r_\Theta^2 \;+\; \sigma\Big),
\]
for a universal constant $c_0>0$ (logarithmic factors absorbed).

\paragraph{Proof sketch.} 
The gradient of the quadratic cost satisfies 
$\nabla c_t(\theta) = 2U^\top V_t U \theta - U^\top U z_t$. 
By boundedness of $z_t$, $V_t$, and $\theta$, the gradient norm is uniformly bounded, with Lipschitz constant $L = O(B_U B_z r_\Theta + B_U^2 B_V r_\Theta^2)$. Standard analysis of the two–point estimator shows that the regret decomposes into (i) initialization error $O(H^2/\eta)$, (ii) variance $O(\eta T d L^2 / \varepsilon^2)$, 
and (iii) smoothing bias $O(LT\varepsilon)$. Optimizing $\eta$ and $\varepsilon$ as in \cite{flaxman2005bandit} yields the $\tilde O(\sqrt{d}\,T^{3/4})$ bound with constants as above. This matches the optimal rate for Lipschitz bandit convex optimization.

\bigskip
\subsection{Algorithmic References}
For completeness, we include the pseudocodes of baseline algorithms compared with ADEPT:
\begin{itemize}
    \item \textbf{Algorithm~1: GDG (Gradient Descent with Bandit Feedback) \cite{flaxman2005bandit}}
    \item \textbf{Algorithm~2: Explore--Exploit Pricing \cite{henaff2019explicit}}
\end{itemize}

Each pseudocode follows the same notation as the main text, with consistent step sizes, smoothing radii, and price update rules.
\medskip
\begin{algorithm}[H]
\caption{GDG: Gradient Descent with Bandit Feedback}
\begin{algorithmic}[1]
\STATE \textbf{Input:} Learning rate $\eta > 0$, smoothing radius $\varepsilon > 0$, horizon $T$
\STATE Initialize $\theta_1 \in [\theta_{\min}, \theta_{\max}]^d$
\FOR{$t = 1, \dots, T$}
    \STATE Sample random unit vector $\xi_t \in \mathbb{R}^d$
    \STATE Construct perturbed attribute price $\tilde{\theta}_t = \theta_t + \varepsilon \xi_t$
    \STATE Set product prices $p_t = U^{\text{add}} \tilde{\theta}_t$
    \STATE Observe realized revenue $r_t = \langle q_t, p_t \rangle$ with demand $q_t$
    \STATE Estimate gradient:
    \[
        \hat{g}_t = \frac{d}{\varepsilon} \, r_t \, \xi_t
    \]
    \STATE Update attribute prices:
    \[
        \theta_{t+1} \leftarrow \Pi_{[\theta_{\min},\theta_{\max}]^d} \Big( \theta_t + \eta \hat{g}_t \Big)
    \]
\ENDFOR
\end{algorithmic}
\end{algorithm}

\begin{algorithm}[H]
\caption{Explore--Exploit}
\begin{algorithmic}[1]
\STATE \textbf{Input:} Horizon $T$, exploration length $m$, number of phases $K = T/m$
\STATE For each phase $k=1,\dots,K$:
\STATE \hspace{0.5cm} \textbf{Exploration:} For $m$ rounds, select random price vectors 
$p_t \in \mathcal{B}_P(R)$ using random $\theta_t$ in $[\theta_{\min}, \theta_{\max}]^d$, 
collect demand $q_t$ and revenues $r_t$
\STATE \hspace{0.5cm} \textbf{Model fit:} Using exploration data, solve a ridge regression problem 
to fit a quadratic surrogate revenue model 
\[
\hat{r}(\theta) \approx a + b^\top \theta - \tfrac{1}{2} \theta^\top C \theta
\]
where $C \succeq 0$
\STATE \hspace{0.5cm} \textbf{Exploitation:} Play the maximizer $\hat{\theta}^\star$ of the surrogate 
model (projected into $[\theta_{\min}, \theta_{\max}]^d$) for the remaining rounds in the phase
\STATE Continue to the next phase
\end{algorithmic}
\end{algorithm}

\end{document}